\documentclass[11pt]{article} % W&CP article
 \usepackage[letterpaper, margin=1in]{geometry}
\usepackage{bbold,color}
\usepackage{amssymb}
\usepackage{amsthm}
\usepackage{natbib}
\usepackage[english]{babel}
\usepackage{subcaption}
\usepackage{graphicx}
\usepackage{float}
\usepackage{amsmath}
\usepackage{enumitem}
\newtheorem{definition}{Definition}[section]
\newtheorem{lemma}{Lemma}[section]

\newtheorem{proposition}{Proposition}[section]

\newtheorem{assumption}{Assumption}[section]
\newcommand{\nX}{\mathbf X_n}
\newcommand{\cX}{\mathcal {X}}
\newcommand{\e}{\epsilon}
\newcommand{\vxi}{\xi}
\newcommand{\x}{x}
\newcommand{\elim}{\text{ep}-\lim}
\newcommand{\hlim}{\text{hypo}-\lim}
\newcommand{\In}{\mathbb I}
\newcommand{\bbP}{\mathbb P}

\newcommand{\bbE}{\mathbb{E}}
\newcommand{\bbD}{\mathbb{D}}

\newcommand{\cQ}{\mathcal{Q}}

\newcommand{\qnv}{q^*(\mathbb \xi|\nX)}

\title{Variational Bayesian Methods for Stochastically Constrained System Design Problems}
 \author{Prateek Jaiswal{$^\star$}, Harsha Honnappa{$^\star$} and Vinayak A. Rao{$^\dag$}}
 \date{$^\star$\{jaiswalp,honnappa\}@purdue.edu School of Industrial Engineering, Purdue University\\$^\dag$\{varao@purdue.edu\} Department of Statistics, Purdue University.}

\begin{document}

\maketitle

\begin{abstract}
We study system design problems stated as parameterized stochastic programs with a chance-constraint set. We adopt a Bayesian approach that requires the computation of a posterior predictive integral which is usually intractable. In addition, for the problem to be a well-defined convex program, we must retain the convexity of the feasible set. Consequently, we propose a variational Bayes-based method to approximately compute the posterior predictive integral that ensures tractability and retains the convexity of the feasible set. Under certain regularity conditions, we also show that the solution set obtained using variational Bayes converges to the true solution set as the number of observations tends to infinity. We also provide bounds on the probability of qualifying a  true infeasible point (with respect to the true constraints) as feasible under the VB approximation for a given number of samples. %Finally, we present simulation results on the optimal design of a queueing system.
\end{abstract}
% Keywords may be removed
%\begin{keywords}
%List of keywords
%\end{keywords}
%\vspace{-5em}
\section{Introduction}
A general system design problem can be formally stated as the following constraint optimization problem
\vspace{-0.5em}
\begin{align*}\tag{TP}
\text{minimize}  &\quad f(\x,\vxi)\\
\text{s.t.}  &\quad g_i(\x,\vxi) \leq 0 , \ i \in \{1,2,3,\ldots, m\},
\end{align*}
where $\x \in \cX \subseteq \mathbb{R}^p$ is the input/control vector in a convex set $\cX$ and $\vxi   \in \Theta \subseteq \mathbb{R}^q$ is the system parameter vector.  The function $f(\x,\vxi): \cX \times \Theta \mapsto \mathbb{R}$ encodes the cost/risk  associated with the given values of  parameter and control variable $\vxi$ and $\x$ respectively. Similarly, the functions $g_i(\x,\vxi): \cX \times \Theta \mapsto \mathbb{R}$ define the constraints on $\vxi$ and $\x$. Under certain regularity conditions on the cost and the constraint functions, and for a given value of the true system parameter $\vxi_0$, the problem (TP) at $\vxi=\vxi_0$ can be solved to obtain the optimal control vector $\x^*$. In practice, the true system parameters are unknown and these parameters must be estimated using observed data.
%, we model a natural phenomenon using \textit{a priori} information about its behavior. A perfect mathematical model must explain the observed input-output relationship accurately. Based on the model and the system parameter values the output can be controlled by deciding the input variable.   The general problem of finding the optimal input, given the true system parameters and objective, can be represented 
%The above problem formulation is common in various application like modeling unit commitment in power systems management~\cite{Zheng2015}, portfolio design in risk management~\cite{Gaivoronski2005}, midterm supply-chain planning at multiple sites~\cite{Mitra2008}, trajectory planning for autonomous vehicles~\cite{Lefko}, staffing problem in designing queueing networks~\cite{Gurvich2010}. 

In this paper, we take a Bayesian approach and model the uncertainty over the parameters $\vxi$ by computing a posterior distribution $\pi(\vxi|\nX)$  for a given prior distribution  $\pi(\vxi)$ and the likelihood $P_{\vxi}(\nX)$ of observing data $\nX$. We approximate the true problem (TP), using the posterior distribution, with the following joint chance-constrained problem:
\begin{align*}\tag{BJCCP}
\text{minimize}  &\quad \bbE_{\pi(\vxi|\nX)}[f(\x,\vxi)]\\
\text{s.t.}  &\quad \pi \left(g_i(\x,\vxi) \leq 0 , \ i \in \{1,2,3,\ldots, m\}|\nX \right) \geq \beta ,  \forall \x \in \cX,
\end{align*}
where $\beta \in (0,1)$ is the specified confidence level desired by the decision maker (DM) based on the requirement, usually $\beta>\frac{1}{2}$. We provide a supporting example~(see Appendix B) to motivate the chance-constraint formulation as opposed to using expectations, in which case the constraints are only satisfied on an average. The unconstrained version of the above problem has been studied as a special case in~\cite{jaiswal2019b}.

In practice, computing posterior distributions is challenging and mostly intractable, and is typically approximated using  Markov Chain Monte  Carlo (MCMC) or Variational Bayesian(VB) methods. MCMC methods has its own drawbacks like poor mixing, large variance, and complex diagnostics, which have been the usual motivation for using VB~\citep{BlKuMc2017}. Here, we provide another important motivation for using VB in the chance-constrained Bayesian inference setting. In particular, we present an example (motivated from~\cite{PenaOrdieres2019SolvingCP}) where a sampling based approach to approximate the chance-constraint convex feasibility set (constraint set) in (BJCCP), results in a non-convex approximation; whereas an appropriate VB approximation \textit{retains its convexity}. Therefore, we approximate (BJCCP) using a VB approximate posterior $\qnv$ to $\pi(\vxi|\nX)$ as:
\begin{align*}\tag{VBJCCP}
\text{minimize}  &\quad \bbE_{\qnv}[f(\x,\vxi)]\\
\text{s.t.}  &\quad q^* \left(g_i(\x,\vxi) \leq 0 , \ i \in \{1,2,3,\ldots, m\}|\nX \right) \geq \beta,  \forall \x \in \cX.
\end{align*}

Under certain regularity conditions, we also show that the optimizers of (VBJCCP) are consistent with those of (TP). More precisely, we show that the solution set obtained in (VBJCCP) converges to the true solution set as the number of observations, $n$ tends to infinity. We also provide bounds on the probability of qualifying a  true infeasible point (with respect to the true constraints) as feasible under the VB approximation for a given number of samples. 
%These results can easily be extended to establish similar properties of (BJCCP), as (VBJCCP) contains (BJCCP) as a special case. 
%Finally, we present simulation results on  the optimal design of a queueing system.
% and provide some simulation results.    
As part of the future work, we want to analyze the  risk-sensitive VB approximation of the (BJCCP), where the risk is quantified as the deviation of the approximate feasibility set from the true. 

\section{Variational Bayes for Chance-Constrained System Design}

Bayesian statistics delineates natural principles to model uncertainty in parameter estimation, using observed data combined with prior knowledge. 
%This uncertainty in the parameter estimation is encoded in the form of a posterior distribution which is proportional to the prior information and the likelihood of the  observed  data. 
Let $\nX= \{X_1,X_2,\ldots,X_n\}$, be $n$ independent and identically distributed samples from the $\mathcal{F}$ measurable random vector ${X(\omega)}$ with support $\Omega \subset \mathbb{R}^{d}$ on probability space $(\Omega, \mathcal{F}, {P_{\vxi}})$, with $P_{\vxi}$ as the associated probability measure, with parameter $\vxi$.

Using the  posterior distribution $\pi(\vxi|\nX)$, we approximate (TP) as a data-driven joint chance-constrained problem, stated formally as:
% we are interested in solving the following joint chance-constrained optimization problem to obtain the optimal input variable $\x \in \cX$:
\begin{align*}\tag{BJCCP}
\text{minimize}  &\quad \bbE_{\pi(\vxi|\nX)}[f(\x,\vxi)]\\
\text{s.t.}  &\quad \pi \left(g_i(\x,\vxi) \leq 0 , \ i \in \{1,2,3,\ldots, m\}|\nX \right) \geq \beta ,  \forall \x \in \cX.
\end{align*}
%where $\Pi(A|\nX):=  \int_{\Theta} \In_{A}(\xi) \pi(\vxi|\nX) d\xi$
and $\beta \in (0,1)$ is the specified confidence level desired by the decision maker (DM) based on the requirement. 
%If there is only one probabilistic constraint, then (BJCCP) is called a single chance constraint problem (BSCCP). 
These are the two significant challenges in solving (BJCCP):

\begin{enumerate}[leftmargin= *]
    \item \textit{Computing the posterior distribution:} While in some cases conjugate priors can be used, this  is not acceptable in most problems; resulting in an intractable computation. 
    %In practice, computing posterior distribution is challenging and mostly intractable. In such cases the DM rely on the  availability of the conjugate priors for Bayesian inference. 
    The posterior intractability is the common motivation for using VB~\citep{BlKuMc2017} and MCMC techniques for approximate Bayesian inference. 
    
    \item \textit{Convexity of the feasibility set:} Ideally, one should expect (BJCCP) to be a convex program to take advantage of the well established  convex solvers. 
    But, even if the posterior distribution is computable, to qualify (BJCCP) as a convex program,  the feasibility set,
    \begin{align}
    \{ \x \in \cX : \pi \left(g_i(\x,\vxi) \leq 0 , \ i \in \{1,2,3,\ldots, m\}|\nX \right) \geq \beta \} 
    \label{eq:fs}
    \end{align} 
    must be convex. 
    %In  which is more difficult than showing that the objective function in JCCP is convex. 
    It might be possible that the above set is not convex even when the underlying constraint  functions $g_i(\x,\vxi), i \in \{1,2,\ldots m\}$  are so (in $\x$) and thus finding a global optimum becomes challenging~\citep{Prkopa1995}.
\end{enumerate}
   
    Note that, if the constraint function has some specified structural regularity and the posterior distribution belongs to a certain class of distributions, then it can be shown that the feasibility set in~\eqref{eq:fs} is convex. For instance, it can be shown that if the constraint functions $g_i(\x,\vxi), i \in \{1,2,\ldots m\}$ are quasi-convex in $(\x,\vec  \xi)$ and the distribution is log-concave then the feasibility set in (BJCCP) is convex (\cite[Chapter 4]{ShDeRu2009} and \cite{prekopa2003}). Also,~\cite{Lagoa2005} showed that if the constraint function $g_i(\x,\vxi)$ is of the form $\{\vec  a^T \x\leq \vec b\}$, where $\vxi=(\vec a^T,\vec b)^T $ and has a symmetric log-concave density then with $\beta>\frac{1}{2}$ the feasibility set in (BJCCP) is convex.
    %\prateek{Add references}

To address the posterior intractability, Monte Carlo (MC) methods offer one way to do approximate Bayesian inference with asymptotic guarantees. However, their asymptotic guarantees are offset by issues like poor mixing, large variance and complex diagnostics in practical settings with finite computational budgets. Apart from these common issues, there is another important reason due to which any sampling-based method can not be used  directly to solve (BJCCP). Using the empirical approximation to the posterior distribution (constructed using  the samples generated from MCMC algorithm) to approximate the chance-constraint feasibility set in (BJCCP), results in a non-convex feasibility set~\citep{PenaOrdieres2019SolvingCP}. To illustrate this, consider the following simple example (modified slightly) of a chance-constraint feasibility set from~\cite{PenaOrdieres2019SolvingCP}. We plot in Figure~\ref{fig:MCVB}(a) the  following chance-constraint feasibility set
%\begin{align}  \left\{\x \in R^2: \mathcal{N} \left( \vec\xi^T\x -1 \leq 0 \bigg | \mathbb \mu= \begin{bmatrix} 
%0 \\  0 
%\end{bmatrix}  ,\Sigma_1= \begin{bmatrix} 
%1 & - 0.025 \\ -0.025 & 1 
%\end{bmatrix}  \right) > \beta \right \} ,
%\label{eq:eq1}
%\end{align}
\begin{align}  \left\{\x \in R^2: \mathcal{N} \left( \vec\xi^T\x -1 \leq 0 | \mathbb \mu= [0,0]^T  ,\Sigma_A= [1,-0.1;-0.1,1]  \right) > \beta \right \} ,
\label{eq:eq1}
\end{align}
and its empirical approximator using 8000 MCMC (Metropolis-Hastings  with 3000 burn-in samples ) samples generated from the underlying correlated multivariate Gaussian distribution. We observe that the  resulting MC approximate feasibility set is non-convex.  

Therefore, due to the posterior intractability and the non-convexity of the feasible region when using sampling approaches, as an alternative, we propose to use Variational Bayes (VB) methods.

The idea behind VB is to approximate the intractable posterior
$\pi(\vxi|\nX)$ with an element $\qnv$ of a simpler \textit{variational family} $\cQ$. Examples of $\cQ$ include the family of Gaussian 
distributions, delta functions, or the family of 
factorized `mean-field' distributions that discard correlations between 
components of $\vxi$. The variational solution $q^*$ is the element of 
$\cQ$ that is
`closest' to $\pi(\vxi|\nX)$, where closeness is usually measured in the 
Kullback-Leibler (KL) sense. Thus, 
\begin{eqnarray}
\qnv := \text{argmin}_{{q} \in \mathcal{Q}} 
\text{KL}({q}(\theta)\|\pi(\vxi|\nX)). \label{eq:vb_opt}
\end{eqnarray}
Using this, we approximate (BJCCP) with, 
\begin{align*}\tag{VBJCCP}
\text{minimize}  &\quad \bbE_{\qnv}[f(\x,\vxi)]\\
\text{s.t.}  &\quad q^* \left(g_i(\x,\vxi) \leq 0 , \ i \in \{1,2,3,\ldots, m\}|\nX \right) \geq \beta,  \forall \x \in \cX,
\end{align*}
where $\beta$ is the confidence level. 
Choosing the approximation to the posterior distribution from a class of `simple' distributions would facilitate in addressing the two critical problems associated with (BJCCP). Besides the  tractability of the posterior distribution, for instance, using the results in~\citet{prekopa2003} and~\cite{Lagoa2005} the choice of  a log-concave family of distributions as the approximating family could retain the convexity of the feasibility set, if the constraint functions have certain structural regularity.

Next,  we show that using the popular mean-field variational family to approximate the correlated multivariate Gaussian distribution in the same example in~\eqref{eq:eq1}, we obtain a smooth and convex approximation to the (BJCCP) feasibility set. First, we compute mean-field approximation $q_A(\vxi)$ and $q_B(\vxi)$ of $\mathcal{N} \left( \vxi  | \mathbb \mu =[0,0]^T  ,\Sigma  \right)$ for fours different covariance matrices $\Sigma$, with fixed variance $\sigma_{11}=\sigma_{22}=1$ but varying covariance $\sigma_{12}=\{-0.1,-0.025,0.025,0.1\}$.
%=\begin{bmatrix} 1 & - 0.025 \\ -0.025 & 1 \end{bmatrix}$ 
%and $\Sigma_B$.
%=\begin{bmatrix} 1 &  0.025 \\ 0.025 & 1 \end{bmatrix}$. 
Then, we plot the respective approximate VB chance-constraint feasibility region  
%\(  \left\{\x \in R^2: q_A \left( \vec\xi^T \x -1 \leq 0  |\vec 0, \sigma^A_1, \sigma^A_2  \right) > \beta \right \} \) and \(  \left\{\x \in R^2: q_B \left( \vxi^T \x -1 \leq 0  |\vec 0, \sigma^B_1,\sigma^B_2 \right) > \beta \right \} \) 
in Figure~\ref{fig:MCVB}. We observe that VB approximation provides a smooth convex approximation to the true  feasibility  set, but it could be outside the true feasibility region if the $ \xi_1$ and $\xi_2$ are positively correlated.
%Consider again the same feasibility set
%where $\mu = (0,0)$ and $\Sigma=$ 
%\(\begin{bmatrix} 
%1 & - 0.05 \\ -0.05 & 1 
%\end{bmatrix}\)
%\vspace{-1.4em}
\begin{figure}[H]
    %\floatconts%
        \begin{subfigure}[b]{0.24\textwidth}%
            \includegraphics[width=\textwidth]{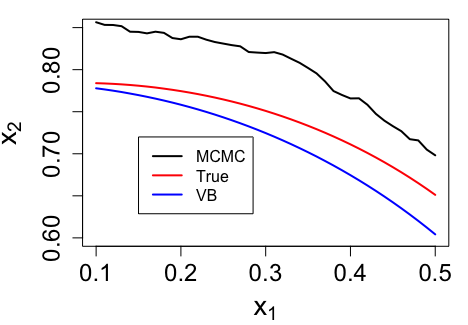}%
        \caption{$\sigma^A_{12}=-0.1$}
       % \label{fig:c1}
        \end{subfigure}
        \begin{subfigure}[b]{0.24\textwidth}%
            \includegraphics[width=\textwidth]{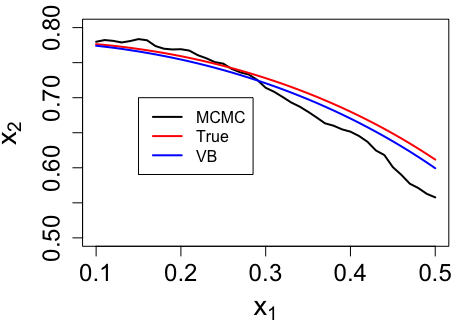}
            \caption{$\sigma^B_{12}=-0.025$}
           % \label{fig:s1}
            \end{subfigure}
        \begin{subfigure}[b]{0.24\textwidth}%
            \includegraphics[width=\textwidth]{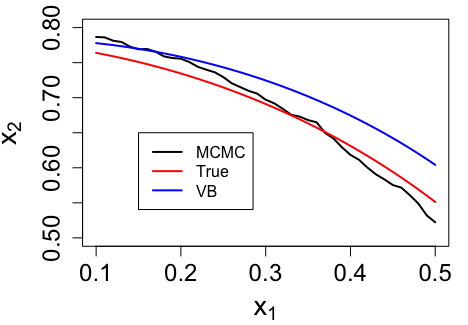}%
        \caption{$\sigma^C_{12}=0.025$}
       %         \label{fig:c2}
        \end{subfigure}
        \begin{subfigure}[b]{0.24\textwidth}%
            \includegraphics[width=\textwidth]{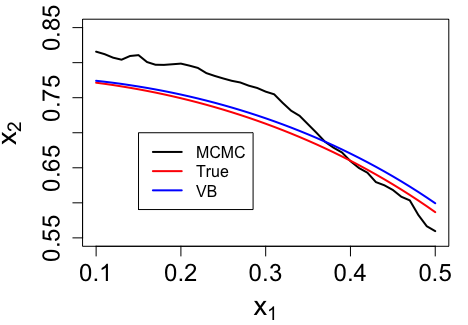}
            \caption{$\sigma^D_{12}=0.1$}
           %     \label{fig:s2}
          \end{subfigure}
\caption{Feasible Region :  True Distribution vs Monte Carlo Approximation (5000 samples) vs. VB (mean field approximation).}
    \label{fig:MCVB}
\end{figure}

%\vspace{-1em}
\subsection{Theoretical properties of (VBJCCP)}

In this section, we establish theoretical guarantees on the approximate optimal solution set $\mathcal S^*_{VB}(\nX)$ obtained using the VB approximation and show that it converges to the optimal solution set $\mathcal S^*$ of (TP) almost  surely in $P_0$. We show similar result for their corresponding optimal values $V^*_{VB}(\nX)$ and $V^*$. %In particular, we show that the solution obtained in (VBJCCP) are consistent and feasible with  high probability.
% We also establish the rate at which the feasible regions  of (VBJCCP) coincides with the true feasible region.  
%\begin{enumerate}
%    \item Is the solution obtained in JCCP and VB-JCCP are consistent and what is their convergence rate to the true optimal solution $x^*$?
%    \item Rate at which the feasible regions  of JCCP and  VB-JCCP coincide with the  true feasible  region.
%\end{enumerate}
The consistency of the approximate solution follows using  techniques from the variational calculus and the consistency of the VB-approximate posterior distribution, which is proved under certain conditions  on the prior distribution, likelihood model, and the variational approximation in~\cite{WaBl2017}. 
%and $g_i(\x,\vxi),i\in\{1,2,\ldots,m\}$
%Intitutively,  assuming $f(\x,\vxi)$ is Lipschitz continuous in $\x$, it is  easy to observe that, the  (VBJCCP) problem converges to the problem (TP) at $\vxi= \vxi_0$ as the number of observation tends to infinty and hence the (VBJCCP) approximate solution should converge to the true solution $\x^*$.
For brevity, we state the following results without any assumptions and  proofs; it will be stated formally in Appendix C.
%Let $\mathcal S^*_{VB} $ and $\mathcal S^*$ denote the optimal solution set of (VBJCCP) and (TP) respectively. Then, we want  to show that
 \vspace{-0.3em}
\begin{proposition}\label{prop:1}
    We show that $V^*_{VB}(\nX) \to  V^*~P_{0}-a.s.$  and  $\bbD(\mathcal S^*_{VB}(\nX),\mathcal S^*) \to  0~P_{0}-a.s. ~\text{as}~n\to\infty$,
    where $\bbD( A, B) := \sup_{\x \in A} \inf_{\vec y \in B} \|\x-\vec  y\|, $
    is the distance between two sets $A$ and $B$. 
    %Let $\x^*_{VB} $ and $\x^*$ denote the optimal solution of (VBJCCP) and (TP) respectively. Then, %i) $\left\{ x^*_{\pi} \subseteq x^* \right\}~P_{0}-a.s. ~\text{as}~n\to\infty$$\left\{ x^*_{\pi} \subseteq x^* \right\}~P_{0}-a.s. ~\text{as}~n\to\infty$ and ii) 
    %$\left\{ \x^*_{VB} \subseteq \x^* \right\}~P_{0}-a.s. ~\text{as}~n\to\infty$.
%    \begin{enumerate}
%        \item $\left\{ x^*_{\pi} \subseteq x^* \right\}~P_{0}-a.s. ~\text{as}~n\to\infty$
%        \item $\left\{ x^*_{VB} \subseteq x^* \right\}~P_{0}-a.s. ~\text{as}~n\to\infty$
%    \end{enumerate}
    % where $A_q^*= \{ \underset{a \in \mathcal {A}}{\arg \inf} \ H_q(a,\mathbf X_n)\}$  and  $A_0^* = \{ \underset {a \in \mathcal {A}}{\arg \inf} \ f(\x,\vxi_0) \}$.
\end{proposition}
\vspace{-0.3em}
In the next result, we show that the solution obtained in (VBJCCP) are feasible with  high probability. Let us define the set where the true constraint $i\in\{1,2,\ldots m\}$ is satisfied as  \( F^i_0:=\{ \x \in \cX :  \{ g_i(\x,\vxi_0) \leq 0 \},   \}, \)
and VB-approximate feasibility set is denoted as
\( \hat{F}_{VB}(\nX) := \{ \x \in \cX : q^* \left(g_i(\x,\vxi) \leq 0 , \ i \in \{1,2,3,\ldots, m\}|\nX \right) \geq \beta \}. \) We prove the next result using the convergence  rate results for VB approximation in~\cite{ZhGa2019}. 
\begin{proposition}\label{prop:2}
    We show that if $\x\in \cX \backslash F_0^i$, then there exists constant $C_i>0$ for each $i\in\{1,2,\ldots m\}$, such that
    %\vspace{-1.2em}
    \( \bbP_0[ \x \in \hat{F}_{VB}(\nX)  ] \leq \frac{C_i}{\beta}(\e_n^2+\eta_n^2), \)
    %\vspace{-1.1em}
    where $\e_n^2 \to 0$ as $n \to \infty$ and 
    %{\small 
    \(\eta_n^2 : = \frac{1}{n} \inf_{q \in \mathcal Q} \bbE_{P_0} \left[  \int_{\Theta}   	q(\vxi)  \log \frac{q(\vxi)} { \pi(\vxi | \nX)   } d\vxi  
    d\vxi \right].  \)
    %}
\end{proposition}

\bibliographystyle{apalike} % outcomment this and next line in Case 1
\bibliography{refs.bib} 

\appendix

\section{A System Design Problem}\label{apd:first}
%\subsection{}
To illustrate the system design problem in (TP) with an example, we model a queueing system and show that the optimal staffing problem aptly fits into the above framework.
%\vspace{1em}
%\begin{example}[Optimal Staffing Problem]
Consider a staffing problem where the decision maker (DM) has to decide the optimal number of servers, $c$, after observing the arrival and service data in an $M/M/c$ queueing system; a queueing system where the inter-arrival times and service times are exponentially distributed (Markovian) with $c$ number  of servers, is denoted as $M/M/c$ queueing system. We assume that the rate parameters of the exponentially distributed inter-arrival and service times distribution are unknown and denoted as $\lambda$ and $\mu$  respectively. Note that $\lambda$ and $\mu$, combined together form the system parameter $\xi = \{\lambda ,\mu\}$  and  the number of servers $c$ is the control/input variable.
%In subsequent paragraphs we discuss the methodology adopted by the DM to decide the number  of extra servers required.
The DM first uses a single server and collects data after the system reaches its `steady state'. Since the DM observed congestion in the queues, he/she decides to employ more servers.  The DM collects $n$ realizations of the random vector $\mathcal{V}:=\{T,S,E\}$, denoted  as $\nX :=\{\mathcal{V}_1,\ldots \mathcal{V}_n\}$ where $T$, $S$, and $E$ are the random variables denoting the arrival, service-start, and service-end time of each customer $i \in \{1,2,\ldots  n\}$ respectively. We also assume that there is no time lag between the two successive states for any customer and the inter-arrival and service times are independent, that is $T_i-T_{i-1}$ is independent of $E_i-S_{i}$ for each  $i \geq 1$. The joint likelihood of the arrival and departure times for $n$ customers is
\[P_{\vxi}(\nX) := \prod_{i=1}^{n} \lambda e^{-\lambda(T_i-T_{i-1})} \mu e^{-\mu(E_i-S_{i})} .\]
%\textbf{Observations}
%\begin{itemize}
%	\item Let $T_i$ be the random variable denoting the time of arrival of the $i^{th}$ customer.
%	\item Let $S_i$ be the random variable denoting the time at which service starts for the $i^{th}$ customer.
%	\item Let $E_i$ be the random variable denoting the time at which service ends for the $i^{th}$ customer.
%	\item We also assume that there is no time lag between the two successive states for any customer.
%	\item Let the DM records the above data for $n$ number of customers, therefore the joint likelihood of the arrival and departure times for $n$ customers is
%	\[P(Z_n|\lambda,\mu) = \prod_{i=1}^{i=n} \lambda e^{-\lambda(T_i-T_{i-1})} \mu e^{-\mu(E_i-S_{i})} .\]
%\end{itemize}
\textbf{Constraint functions:} The DM chooses the number of servers $c$ to maintain a constant measure of congestion. Congestion is usually measured as $1-W_q(c,\lambda,\mu)$, where $W_q(c,\lambda,\mu)$ is the probability that the customer did not wait in the queue. The closed-form expression for $1-W_q(c,\lambda,\mu)$ is  known to be (see~\cite{Gross2008})
\[ 1-W_q(c,\lambda,\mu) = \frac{r^c}{c!(1-\rho)} \Bigg/ \left(\frac{r^c}{c!(1-\rho)} + \sum_{t=0}^{c-1}\frac{r^t}{t!} \right), \]
where $r= \frac{\lambda}{\mu} \text{ and } \rho =\frac{r}{c}$ with $\rho<1$. $\rho $ is also known as \textit{traffic intensity} and for a stable queue $\rho<1$. DM fixes $\alpha$, which is maximum desired fraction of customers delayed in queue and the smallest $c$ is chosen that satisfies:
\[  ( \alpha - \{1-W_q(c,\lambda,\mu)\} ) > 0 \text{ and } (c \mu - \lambda) > 0.   \]
Referring to the queueing literature,  we will use the term the Quality of Service\textsc{(QoS)} constraint for the first constraint. The corresponding constraint optimization problem is
\begin{align*}\tag{TP-Q}
\text{minimize} &\quad c \\
\text{s.t.}  &\quad ( \alpha - \{1-W_q(c,\xi)\} ) > 0 \quad \textsc{(QoS)},\\
&\quad (c \mu - \lambda) > 0.
\end{align*}
The above staffing problem and its variations are well studied in the  queueing literature; interested reader may refer to~\cite{Gans2003} and~\cite{Aksin2009}.
%\end{example}

\section{Other Data-Driven Approaches to solve (TP)}\label{apd:second}
Since the system parameters are unknown in  practice, these are usually estimated using the observed data $\nX$. The simplest approach could be to substitute the maximum likelihood estimates(MLE) $\hat \xi(\nX)$ of the parameters $\xi$ in the (TP) and solve the following  approximate  problem:
\begin{align*}\tag{TP-MLE}
\text{minimize}  &\quad f(x,\hat \xi)\\
\text{s.t.}  &\quad g_i(x,\hat \xi) \leq 0 , \ i \in \{1,2,3,\ldots, m\}.
\end{align*}
We solved the queueing  staffing problem in~Appendix A using the  MLE approach on simulated data ($n$ observations, with $n$ in  50-400 in increments of 50) and computed the approximate optimal number of servers denoted  as  $C^n_{MLE}$. We repeated this experiment  over 100 sample paths and computed $\phi(C^n_{MLE})$, the  fraction of experiments $C^n_{MLE}$ violates the QoS constraint. Table~\ref{tab:MLE_Staffing} shows that the QoS constraint is violated in over 50\% of the experiments.
\begin{table}[h]
    \centering
    \resizebox{0.85\textwidth}{!}{%
        \begin{tabular}{|l|l|l|l|l|l|l|l|l|}
            \hline
            $n$ & 50 & 100 & 150 & 200 & 250 & 300 & 350 & 400 \\ \hline
            $\phi(C^n_{MLE})$& 0.52 & 0.56 & 0.57 & 0.57 & 0.61 & 0.62 & 0.58 & 0.56 \\ \hline
        \end{tabular}%
    }
\vspace{0.5em}
    \caption{Fraction of times $C^n_{MLE}$ violates QoS constraint. $\lambda_0= 16,\mu_0=4,\alpha=0.37 $.}
    \label{tab:MLE_Staffing}
\end{table}  

It is anticipated that the MLE approach is unable to capture the  uncertainty in parameter estimation therefore an alternative method is proposed using forecasting  techniques. In this approach  first, the uncertainty over the parameter estimation is captured by forecasting a probability distribution $P(\xi)$ over the system parameters and then the forecast distribution is used to solve the (TP) problem using one of the following two formulations:
\begin{itemize}
    \item Average-Constraint(AC) 
    \begin{align*}\tag{TP-FAC}
    \text{minimize}  &\quad \bbE_{P}[f(x, \xi)]\\
    \text{s.t.}  &\quad \bbE_{P}[g_i(x, \xi)] \leq 0 , \ i \in \{1,2,3,\ldots, m\},
    \end{align*}
    
    \item Chance-Constraint(CC) 
    \begin{align*}\tag{TP-FCC}
    \text{minimize}  &\quad \bbE_{P}[f(x, \xi)]\\
    \text{s.t.}  &\quad P \left\{ g_i(x, \xi) \leq 0 , \ i \in \{1,2,3,\ldots, m\} \right\} > \beta,
    \end{align*}
    where $\beta\in (0,1)$ is the \textit{confidence level}. 
\end{itemize}

Now consider a simple example from~\cite{Hong2011}, where the true problem is to find $c^*=\min \{c: \xi-c \leq 0 \} $. Since, $\xi$ is unknown the DM uses data to forecast that $\xi\sim \mathcal{N}(\cdot|0,1)$ is normally distributed. Using the AC formulation, notice that the approximate optimal solution $c_A^*=\min \{c: \bbE_{\xi}[\xi]-c \leq 0 \} = 0$  and $\bbP_{\vxi}\{\xi \geq c_A^* \} =0.5$. 
%$\{\xi_1,\xi_2,\ldots \xi_m\}$ be $m$ iid samples from 
The above simple example shows that AC optimal solution could violate the constraint 50\% of the times. On the other hand, CC formulation enables the DM to ensure that the approximate optimal solution satisfy the constraints with higher confidence by setting a higher confidence level($\beta$). In forecasting approach, the DM needs to forecast each time the new data is collected. We propose a principled  data-driven approach using Bayesian methods, wherein we combine forecasting and optimization. A similar approach has also been discussed in~\cite{Aktekin2016} to solve the $M/M/c$ staffing problem with  abandonment, but crucially relies on the  availability of conjugate priors. 

%We introduce a  Variational Bayesian(VB)  approach that aid in solving the above problem for any choice of priors other than just the conjugate priors. Using VB we approximate the posterior chance-constrained problem.  We also prove asymptotic guarantees on VB based approach.
%This is the second appendix.
\section{Proofs} 
\subsection{Proof of Proposition~\ref{prop:1}  }
%\begin{proposition}

\begin{assumption}\label{ass:Cath}
    We assume that the function  $f(\x, \vxi)$ and $g_i(\x,\vxi), \forall i\in\{1,2,\ldots  m\}$ are \textit{Carath\'eodory} functions; that is $f(\x,\cdot)$ and $g_i(\x,\cdot)$  are measurable for every $\x \in \cX$, and $f(\cdot,\vxi)$ and $g(\cdot,\vxi)$ are continuous for almost every $\vxi \in \Theta$. We also assume that $f(\cdot,\vxi)$ is locally Lipschitz continuous in $\x$ for almost every $\vxi \in \Theta$ and  $f(\x,\cdot)$ is uniformly integrable with respect to any $q \in \cQ$, the  variational family.
    \end{assumption}
Next define an indicator function $\In_{(-\infty,0]}(t) : =  1 \text{ if $t \leq 0$ and } 0 \text{ if $t > 0$ } $. 

\begin{lemma}\label{lem:pw}
    We show that  for each $\x \in \cX$
     \[ \lim_{n \to \infty} q^* \left( \prod_{i=1}^{m}\In_{(-\infty,0]}(g_i(\x,\vxi)) |\nX \right) = \prod_{i=1}^{m}\In_{(-\infty,0]}(g_i(\x,\vxi_0))~ P_{0}-a.s.\]
\end{lemma}
\begin{proof}
     Recall the result in~\cite{WaBl2017} that the VB approximate posterior $\qnv$ is consistent; that  is for every $\eta>0$.
    %that is for any $\eta>0$,  the NV variational posterior  satisfies
    \begin{align} \lim_{n \to \infty} \int_{\|\vxi-\vxi_0\|>\eta} \qnv d\vxi  = 0 ~ P_{0}-a.s. 
    \label{eq:p1}
    \end{align}
    Observe that for any $\x \in \cX$ and $\eta>0$,
    \begin{align}
    \nonumber
    q^* \left(\prod_{i=1}^{m}\In_{(-\infty,0]}(g_i(\x,\vxi)) |\nX \right) &= \int_{\Theta}\prod_{i=1}^{m}\In_{(-\infty,0]}(g_i(\x,\vxi)) \qnv d\vxi
    \\
    = \int_{\|\vxi-\vxi_0\|>\eta} \prod_{i=1}^{m}\In_{(-\infty,0]}&(g_i(\x,\vxi)) \qnv d\vxi + \int_{\|\vxi-\vxi_0\|\leq \eta}\prod_{i=1}^{m}\In_{(-\infty,0]}(g_i(\x,\vxi)) \qnv d\vxi.
    \label{eq:p2}
    \end{align}
    Observe that, the result in~\eqref{eq:p1} combined with the fact that the first term in~\eqref{eq:p2} is always positive and  bounded, implies that  $\lim_{n \to \infty} \int_{\|\vxi-\vxi_0\|>\eta}\prod_{i=1}^{m}\In_{(-\infty,0]}(g_i(\x,\vxi)) \qnv d\vxi = 0~ P_{0}-a.s$.
    Now taking limits on either side of~\eqref{eq:p2}, we have
    \begin{align}
    \nonumber
    \lim_{n \to \infty} q^* \left(\prod_{i=1}^{m}\In_{(-\infty,0]}(g_i(\x,\vxi)) |\nX \right) &= \lim_{n \to \infty} \int_{\|\vxi-\vxi_0\|\leq \eta}\prod_{i=1}^{m}\In_{(-\infty,0]}(g_i(\x,\vxi)) \qnv d\vxi ~ P_{0}-a.s,
    \\
    &= \prod_{i=1}^{m}\In_{(-\infty,0]}(g_i(\x,\vxi_0)) ~ P_{0}-a.s
    \label{eq:p3}
    \end{align}
    and the lemma follows.
    \end{proof}

Next we define hypo-convergence and epi-convergence of a sequence of function $\{h_k(\x)\}$ to $h(\x)$.
\begin{definition}[Hypo-convergence]
    A sequence of functions $\{h_k(\x)\}$ hypo-converges to $h(\x)$; that is $\hlim_{n \to \infty} h_k(
    \x) = h(\x)$, if 
    \begin{enumerate}
        \item for every $\x_k\to \x$, $\limsup_{ k \to  \infty} h_k(\x_k) \leq h(\x)     $, and
        \item there exists a sequence $\x_k\to \x$, such that $\liminf_{ k \to  \infty} h_k(\x_k) \geq h(\x) $.
    \end{enumerate}
\end{definition}

\begin{definition}[Epi-convergence]
    A sequence of functions $\{h_k(\x)\}$ epi-converges to $h(\x)$; that is $\elim_{n \to \infty} h_k(
    \x) = h(\x)$, if 
    \begin{enumerate}
        \item for every $\x_k\to \x$, $\liminf_{ k \to  \infty} h_k(\x_k) \geq h(\x)     $, and
        \item there exists a sequence $\x_k\to \x$, such that $\limsup_{ k \to  \infty} h_k(\x_k) \leq h(\x) $.
    \end{enumerate}
\end{definition}

\begin{lemma}\label{lem:of}
    Under Assumption~\ref{ass:Cath}, we show that,
    \begin{enumerate}
        \item for each $\x \in \cX$,
    \( \lim_{ n \to \infty} \bbE_{\qnv}[f(\x,\vxi)]  = f(\x, \vxi_0)~ P_{0}-a.s.\)
        \item and,
    $\elim_{ n \to \infty} \bbE_{\qnv}[f(\x_n,\vxi)] = f(\x_0,\vxi_0)~ P_{0}-a.s$. 
        \end{enumerate}
\end{lemma}
\begin{proof}
    Due to Assumption~\ref{ass:Cath}, both the results above are a direct consequence of the result in~\cite[Theorem 3.7]{Dupacova1988}. 
    \end{proof}

\begin{lemma}\label{lem:hypo}
    We show that under Assumption~\ref{ass:Cath}, $q^* \left(\prod_{i=1}^{m}\In_{(-\infty,0]}(g_i(\x,\vxi)) |\nX \right)$ hypo-converges to $\prod_{i=1}^{m}\In_{(-\infty,0]}(g_i(\x,\vxi_0))$ $P_{0}-a.s$ as $n \to \infty$; that is
    \begin{align}
    \hlim_{n \to \infty} q^* \left(\prod_{i=1}^{m}\In_{(-\infty,0]}(g_i(\x,\vxi))  |\nX \right)
    &= \prod_{i=1}^{m}\In_{(-\infty,0]}(g_i(\x,\vxi_0)) ~ P_{0}-a.s.
    %\label{eq:p4};
    \end{align}
    \end{lemma}
\begin{proof}
    Since by Assumption~\ref{ass:Cath} each $g_i(\x,\vxi_0)$ is continuous in $\x$, therefore $\In_{(-\infty,0]}(g_i(\x,\vxi_0))$ is  upper-semicontinuous(USC) in $\x$ because $\In_{(-\infty,0]}(\cdot)$ is USC. Also, since the product of non-negative USC functions are also USC, it follows that $\prod_{i=1}^{m}\In_{(-\infty,0]}(g_i(\x,\vxi_0))$  is  USC. Similarly, since by assumption $g_i(\x,\vxi)$ is \textit{Carath\'eodory} function, therefore $\prod_{i=1}^{m}\In_{(-\infty,0]}(g_i(\x,\vxi))$ is a random upper-semicontinuous function~\citep{Dupacova1988}. Now, using the reverse Fatou's Lemma, for any $\x_k \to \x_0$ 
    \begin{align}
    \nonumber
    \limsup_{\x_k \to \x_0} \int_{\Theta}\prod_{i=1}^{m}\In_{(-\infty,0]}(g_i(\x,\vxi)) \qnv d\vxi &\leq  \int_{\Theta} \limsup_{\x_k \to \x_0} \prod_{i=1}^{m}\In_{(-\infty,0]}(g_i(\x,\vxi)) \qnv d\vxi
    \\
    & \leq \int_{\Theta} \prod_{i=1}^{m}\In_{(-\infty,0]}(g_i(\x_0,\vxi)) \qnv d\vxi,
    \end{align}
    therefore $\int_{\Theta} \prod_{i=1}^{m}\In_{(-\infty,0]}(g_i(\x,\vxi)) \qnv d\vxi$ is also upper-semicontinuous in $\x$. Also, since $-\prod_{i=1}^{m}\In_{(-\infty,0]}(g_i(\x,\vxi)) \in \{-1,0\}$ is a bounded random lower-semicontinuous function and $\qnv \Rightarrow \delta_{\vxi_0}~P_{0}-a.s$~\citep{WaBl2017}, Theorem 3.7 in~\cite{Dupacova1988} implies that 
    \begin{align}
    \hlim_{n \to \infty} q^* \left(\prod_{i=1}^{m}\In_{(-\infty,0]}(g_i(\x,\vxi))  |\nX \right)
    &= \prod_{i=1}^{m}\In_{(-\infty,0]}(g_i(\x,\vxi_0)) ~ P_{0}-a.s
    \label{eq:p4};
    \end{align}
    and the  result  follows.
    \end{proof}
\begin{proof}[Proof of Proposition~\ref{prop:1}]
    %Let $\mathcal S^*_{VB} $ and $\mathcal S^*$ denote the optimal solution set of (VBJCCP) and (TP) respectively. 
    %Then, we want  to show that
    %$\bbD(\mathcal S^*_{VB},\mathcal S^*) \to  0~P_{0}-a.s. ~\text{as}~n\to\infty$,
    %where $\bbD( A, B) := \sup_{\x \in A} \inf_{\vec y \in B} \|\x-\vec  y\|, $
 %is the distance between two sets $A$ and $B$. first assume that $m=1$. 
    We will first show that the assertion of the theorem is  true for $m=1$. Recall $\mathcal S^*_{VB}(\nX) $ is the solution of (VBJCCP):
%    \begin{align*}\tag{VBJCCP}
%    \text{minimize}  &\quad \bbE_{\qnv}[f(\x,\vxi)]\\
%    \text{s.t.}  &\quad q^* \left(g(\x,\vxi) \leq 0  |\nX \right) \geq \beta,  \forall \x \in \cX,
%    \end{align*}
    and $\mathcal S^* $ is the solution of (TP).
%    \begin{align*}\tag{(TP)}
%    \text{minimize}  &\quad f(\x,\vxi_0)\\
%    \text{s.t.}  &\quad g(\x,\vxi_0) \leq 0 ,  \forall \x \in \cX.
%    \end{align*}
    %For simplicity, let us  first assume that $m=1$.

      %The second inequality follows since... 

    Now observe that, since both $q^* \left(g(\x,\vxi) \leq 0  |\nX \right)$ and $\In_{(-\infty,0]}(g(\x,\vxi_0))$ are upper- semicontinuous their corresponding super-level sets are closed; and if $\cX$ is bounded than the corresponding feasible sets are also compact. Also, if the the corresponding feasibility sets  are non-empty then the corresponding optimal sets $\mathcal S^*_{VB}(\nX)$ and $\mathcal S^*$ are also non-empty.
    
    Next let us assume that there exists a true solution $\x^*$ of (TP) which lies in the interior of $\cX$, that is for any $\e>0$, there is $\x \in \cX$ such that $\|\x-\x^*\|<\e$ and $g(\x, \vxi_0)\leq 0$. It implies that there exists a sequence $\{\x_k\}\subset \cX$  such that $\x_k \to \x^*$ as $k \to \infty$ and $g(\x_k, \vxi_0)\leq 0$ for all $k\geq 1$.
    Now fix $\x\in\cX $ such that $g(\x, \vxi_0)\leq 0$. Since, due to our result in~Lemma~\ref{lem:pw} $q^* \left(g(\x,\vxi) \leq 0  |\nX \right)$ converges pointwise to $\In_{(-\infty,0]}(g(\x,\vxi_0))~ P_{0}-a.s$, therefore there exists an $n_0$ such that for all $n\geq n_0$, we have $q^* \left(g(\x,\vxi) \leq 0  |\nX \right) \geq \beta$. 
    %Since $q^* \left(g(\x,\vxi) \leq 0  |\nX \right) $ is upper-semicontinuous the solution set $\mathcal S_{VB}^*$ is non-empty for large enough  $n$. 
    Hence for all $n\geq n_0$, $\x$ is a feasible solution  of (VBJCCP) and therefore $\bbE_{\qnv}[f(\x,\vxi)]  \geq V_{VB}^*(\nX)$. Taking $\limsup$ on either sides, we obtain
    \[\limsup_{ n \to \infty} V_{VB}^*(\nX) \leq \limsup_{ n \to \infty} \bbE_{\qnv}[f(\x,\vxi)]  = f(\x, \vxi_0)~ P_{0}-a.s,\]
    where the last inequality follows from Lemma~\ref{lem:of}~(1).
    %the fact that $f(\cdot,\cdot)$ is a \textit{Carath\'eodory} function and $\qnv \Rightarrow \delta_{\vxi_0}~P_{0}-a.s$~\cite[Theorem 3.7]{Dupacova1988}. 
    Now, since $\x$ can be chosen arbitrarily close to  $\x^*$, it follows that
    \begin{align}
    \limsup_{ n \to \infty} V_{VB}^*(\nX) \leq f(\x^*, \vxi_0) = V^*~ P_{0}-a.s.
    \label{eq:p5}
    \end{align}
    Next, let $\hat\x_n \in \mathcal{S}_{VB}^*$; that is $\hat\x_n \in  \cX$, $q^* \left(g(\hat\x_n,\vxi) \leq 0  |\nX \right) \geq \beta $ and $V_{VB}^*(\nX) = \bbE_{\qnv}[f(\hat\x_n,\vxi)]$. Since $\cX$ is compact, we assume that  $\hat\x_n \to \x^*~P_{0}-a.s$. Due to~Lemma~\ref{lem:hypo}, $q^* \left(g(\x,\vxi) \leq 0  |\nX \right)$ hypo-converges to $\In_{(-\infty,0]}(g(\x,\vxi_0))$ $P_{0}-a.s$ as $n \to \infty$, therefore we have
    \begin{align}
    \limsup_{n \to \infty} q^* \left(g(\hat\x_n,\vxi) \leq 0  |\nX \right) \leq  \In_{(-\infty,0]}(g(\x^*,\vxi_0)).
    \label{eq:p7}
    \end{align}
    Now using the fact that $q^* \left(g(\hat\x_n,\vxi) \leq 0  |\nX \right) \geq \beta $ for every $n\geq 1$, it follows from~\eqref{eq:p7} that $\x^*$ is a feasible point of (TP), since $\limsup_{n \to \infty} q^* \left(g(\hat\x_n,\vxi) \leq 0  |\nX \right) \geq \beta$ implies $\In_{(-\infty,0]}(g(\x^*,\vxi_0))\geq \beta$ and $\beta\in (0,1)$. Therefore, it follows that
    $f(\x^*,\vxi_0)\geq V^*$. Since, due to Lemma~\ref{lem:of}~(2), $\liminf_{n\to \infty} \bbE_{\qnv}[f(\hat\x_n,\vxi)] \geq f(\x^*,\vxi_0)~ P_{0}-a.s$, it follows that
    \begin{align} 
    \liminf_{n\to \infty} V_{VB}^*(\nX) \geq V^*~ P_{0}-a.s.
    \label{eq:p6}
    \end{align}
    Hence, it follows from~\eqref{eq:p5} and~\eqref{eq:p6} that $V_{VB}^*(\nX) \to V^*~P_{0}-a.s$ and it also follows that  $x^*$ is the  true solution of (TP), therefore $\bbD(\mathcal S_{VB}^*(\nX),\mathcal S^*) \to 0~P_{0}-a.s$.
    The above arguments can be easily generalized for the general case with $m$ number of constraints.
    %since $\hat\x_n$ is a feasible point of VB-JCCP then it is a feasible point of (TP) as well.
 \end{proof}
   
\subsection{Proof of Proposition~\ref{prop:2}}
    \begin{proof}
        Using Markov's inequality observe that for any  $\x\in \cX$,
        \begin{align}
        \nonumber
        \bbP_0[ q^* \left(g_i(\x,\vxi) \leq 0 , \ i \in \{1,2,3,\ldots, m\}|\nX \right) \geq \beta ] &\leq \frac{1}{\beta} \bbE_0[ q^* \left(\cap_{i=1}^m \{ g_i(\x,\vxi) \leq 0\}|\nX \right) ]\\
        \leq   \frac{1}{\beta} &\bbE_0[ q^* \left( \{ g_i(\x,\vxi) \leq 0\}|\nX \right) ]  ~\forall i \in \{1,\ldots,m\}.
        \label{eq:eq2} 
        \end{align}
        Fix $i \in \{1,2\ldots,m\}$. Since $\x\in \cX \backslash F^i_0$ implies that $\x \in \{ g_i(\x,\vxi_0) > 0 \}$, it follows that 
        \[\{ g_i(\x,\vxi) \leq 0\} \subseteq \{ g_i(\x,\vxi) < g_i(x,\xi_0)\}. \]
        Therefore, for all $\x\in \cX \backslash F^i_0$ and any $i \in \{1,\ldots,m\}$, it  follows from~\eqref{eq:eq2} that
        \begin{align}
        \nonumber
        \bbP_0[ q^* \left(g_i(\x,\vxi) \leq 0 , \ i \in \{1,2,3,\ldots, m\}|\nX \right) \geq \beta ] &\leq \frac{1}{\beta} \bbE_0[ q^* \left(\cap_{i=1}^m \{ g_i(\x,\vxi) \leq 0\}|\nX \right) ] \\
        & \leq \frac{1}{\beta} \bbE_0[ q^* \left(\{ g_i(\x,\vxi) < g_i(\x,\vxi_0)\}|\nX \right) ].
        \label{eq:p21}
        \end{align}
        Now using Theorem 2.1 in~\cite{ZhGa2019}, it follows that for each $i \in \{1,2,3,\ldots, m\}$ if \[L^i_n(\theta,\theta_0):= n \sup_{\x \in \cX} \In_{(0,\infty)}( g_i(\x,\vxi_0) - g_i(\x,\vxi))\] satisfies assumption (C1), then there exists a constant $C_i$  for each $i \in \{1,2,3,\ldots, m\}$ such that 
          \[\bbE_0[ q^* \left(\{ g_i(\x,\vxi) < g_i(\x,\vxi_0)\}|\nX \right) ] \leq C_i (\e_n^2+\eta_n^2),\]
          where $\eta_n^2:= \frac{1}{n} \inf_{q \in \mathcal Q} \bbE_{P_0} \left[  \int_{\Theta}   	q(\vxi)  \log \frac{q(\vxi)} { \pi(\vxi | \nX)   } d\vxi  
      d\vxi \right]$. Now observe that, using the above result in~\eqref{eq:p21} directly proves the assertion of the proposition.
    \end{proof}

\end{document}